\documentclass{article}

\PassOptionsToPackage{numbers, compress}{natbib}
\usepackage[final]{neurips_2026}

\usepackage[utf8]{inputenc} 
\usepackage[T1]{fontenc}    
\usepackage{hyperref}       
\usepackage{url}            
\usepackage{booktabs}       
\usepackage{amsfonts}       
\usepackage{nicefrac}       
\usepackage{microtype}  
\usepackage{comment}
\usepackage{xcolor}         
\usepackage{graphicx}
\usepackage{float}
\usepackage{amsmath,amssymb,amsthm,mathtools,mathrsfs,bm}
\usepackage{enumitem}

\graphicspath{{./images/}}

\hypersetup{
  colorlinks=true,
  linkcolor=blue,
  citecolor=blue,
  urlcolor=blue
}

\newcommand{\safeincludegraphics}[2][]{%
  \IfFileExists{#2}{\includegraphics[#1]{#2}}{%
    \IfFileExists{./images/#2}{\includegraphics[#1]{#2}}{\fbox{Missing figure}}}%
}

\title{Does Order Matter: Connecting the Law of Robustness to Robust Generalization}

\author{%
  Mihir More\\
  Ashoka University, Sonepat, Haryana\\
  \And
  Aritra Das\\
  Ashoka University, Sonepat, Haryana\\
  Truth Audit Labs 
  \And
  Jaee Ponde\\
  Ashoka University, Sonepat, Haryana\\
  \And
  Himadri Mandal\\
  Indian Statistical Institute, Kolkata\\
  \And
  Vishnu Varadarajan\\
  Ashoka University, Sonepat, Haryana\\
  \And
  Debayan Gupta\\
  Ashoka University, Sonepat, Haryana\\
  Truth Audit Labs
}
\theoremstyle{plain}
\newtheorem{theorem}{Theorem}
\newtheorem{lemma}{Lemma}

\newtheorem{corollary}{Corollary}

\theoremstyle{definition}
\newtheorem{definition}{Definition}
\theoremstyle{remark}
\newtheorem{remark}{Remark}

\newcommand{\R}{\mathbb{R}}

\newcommand{\norm}[1]{\left\lVert #1 \right\rVert}

\newcommand{\X}{\mathcal{X}}

\newcommand{\diam}{\mathrm{diam}}

\usepackage{todonotes}

\begin{document}
\maketitle
\begin{abstract}

Bubek and Selke~\citep{bubeck2021universal} propose the connection between the 
law of robustness and robust generalization error as a 
\textit{fantastic} open problem. The Law of Robustness states 
that overparameterization is necessary for models to interpolate 
robustly, i.e., the function is required to be Lipschitz. Wu et al.~\citep{wu2023law} extend this law for arbitrary data distributions, proving L = \(\Omega(n^{1/d})\). Robust 
generalization, on the other hand, asks if small robust training loss implies small 
robust test loss. This can be studied using statistical learning techniques like Rademacher complexities. Accordingly, a bound on the Rademacher complexitiy of the robust loss class implies a bound on the Lipschitzness of the function class. We use this to explicitly connect the two for arbitrary data distributions.
i) We prove that the order of the Lipschitz bound remains the same when considering global Rademacher complexity of robust loss classes.
ii) At the local scale, that is, for subsets of functions with small empirical error, the order of the Lipschitz bound changes with the perturbation radius $\rho$ and localized concentration term $\sqrt {r/n}$.

\end{abstract}

\section{Introduction}
Deep neural networks have demonstrated exceptional performance across different machine learning applications~\citep{lecun2015deep}. One of the fundamental methodologies of deep learning involves overparameterized models. These networks are highly expressive and are often easy to optimize to (near) zero training error with SGD. \citep{Mingard_2025} In many settings, making the model even larger can improve test performance beyond the interpolation threshold, a phenomenon also known as double descent~\citep{nakkiran2019deepdoubledescentbigger}. Bubek et al.~\citep{bubeck2021law} provide a theoretical explanation for why overparameterization is necessary: they introduce a sharp law of robustness for two layer neural networks showing that achieving robust interpolation characterized by a small Lipschitz constant requires more parameters than interpolation alone. Bubek and Selke~\citep{bubeck2021universal} then prove a universal version under isoperimetric covariate distributions, showing that smooth interpolation for $O(1)$ Lipschitz requires $p \gg nd$ parameters. Furthermore, Wu et al.~\citep{wu2023law} extend the theory beyond isoperimetry, providing robust-interpolation lower bounds for arbitrary bounded-support distribution. Despite their remarkable success, deep networks remain vulnerable to subtle perturbations, as revealed through the existence of adversarial samples~\citep{goodfellow2015explainingharnessingadversarialexamples}. The most common mitigation strategy is adversarial training, where models are trained on adversarially perturbed examples to improve robustness~\citep{shafahi2019adversarialtrainingfree}. Other methods such as margin-based measures, flatness-based measures, and gradient-based measures have also been utilized to improve adverserial robustness~\citep{scaman2019lipschitzregularitydeepneural}. Another line of work pursues certified robustness by constraining the Lipschitz constant of the model, which provides formal guarantees on the maximum possible perturbation impact~\citep{NEURIPS2023_98a5c047}. 

However, none of these methods explicitly connect these two different notions of robustness, i.e., the Law of Robustness to robust generalization. 
Bubek and Selke~\citep{bubeck2021universal} refer to this connection as a ``fantastic open problem.'' We study this problem through the lens of Rademacher complexity, which provides distribution-dependent uniform bounds on the deviations of the induced loss class. 

We first show that, under the assumption of overfitting, we can bound the robust generalization error. We then apply global Rademacher bounds to obtain upper and lower bounds on the complexity of the loss class. This allows us to derive a bound on the Lipschitz constant of the robust loss, which we observe does not change its order and remains $\Omega(n^{1/d})$, consistent with the distribution-free setting in \citep{wu2023law}.

However, global Rademacher complexity is too coarse for interpolation and low-error regimes, because it measures the complexity over the entire loss class rather than the small-risk region selected by the learning process ~\citep{bartlett2005local,koltchinskii2006local}. Local Rademacher complexity allows us to compute the complexity of subclasses with small empirical risk, thus providing tighter bounds. Covering numbers can be used to relate metric entropy estimates to local Rademacher complexity bounds ~\citep{LeiDingBi2015}. Motivated by this, we polynomially bound the metric entropy of Lipschitz function classes using covering numbers and graph theory. We use this result to bound the metric entropy of the robust square loss, and then prove that polynomial entropy implies a local bound on the Rademacher complexity. 

We extend our previously obtained lower bounds on the Robust Generalization gap in the global setting (where we showed that the order of the Lipschitz constant does \textit{not} change), to the local setting. We now find that the bound on the Lipschitz constant of the robust loss \textit{is dependent} upon the perturbation radius $\rho$ and localized concentration term $\sqrt {r/n}$. So we answer the titular question, \textit{does order matter}? It turns out this depends on how closely one looks!

\section{Related Work}
\subsection{Adversarial robustness, adversarial training, and Lipschitz}
Modern neural networks are vulnerable to \emph{adversarial examples}, a phenomenon first identified by \citep{szegedy2013intriguing,goodfellow2014explaining}. 
The most widely used defense is \emph{adversarial training}, which minimizes a worst-case objective over bounded perturbations \citep{madry2018towards}, with principled variants such as TRADES making explicit the trade-off between standard and robust accuracy \citep{zhang2019trades}.
Furthermore, parameter-free evaluation suites such as AutoAttack are commonly used for reliable comparison due to robustness being highly sensitive to the threat model
\citep{croce2020reliable}. Another line of work deals \emph{certified robustness} by controlling (global or local) Lipschitz.
Bounding a model's Lipschitz constant yields worst-case guarantees on output variation under input
perturbations, motivating Lipschitz regularization and Lipschitz-constrained architectures. Examples such as Parseval networks \citep{cisse2017parseval}, Lipschitz-margin training
for scalable certification \citep{tsuzuku2018lipschitz}, and more recent work emphasizing
expressive Lipschitz networks for improved certified robustness \citep{zhang2022rethinking}. However, computing tight global Lipschitz constants for deep neural networks is difficult, so much work has focused on computable bounds and approximations.
CLEVER relates robustness evaluation to local Lipschitz estimation \citep{weng2018clever}, while
optimization-based methods provide tighter upper bounds for neural networks \citep{fazlyab2019efficient}.
Zuhlke and Kudenko~\citep{zuhlke2025lipschitzsurvey} provide a comprehensive survey unifying Lipschitz constants, adversarial attacks, robustness guarantees, and estimation methods.

\subsection{Overparameterization and laws of robustness}
Recent theory connects \emph{model size} and \emph{smooth interpolation} through lower bounds on the
Lipschitz constant of interpolating predictors.
Bubek et al.~\citep{bubeck2021law} formulate and prove a sharp trade-off for two-layer networks, suggesting that
overparameterization is necessary to achieve small Lipschitz constant while fitting the data.
Bubek and Selke~\citep{bubeck2021universal} establish a \emph{universal} version under isoperimetric-type covariate distributions,
showing that smooth interpolation can require substantially more parameters than interpolation alone. Wu et al.~\citep{wu2023law} extend this direction beyond isoperimetry and derive distribution-free
robust-interpolation lower bounds, including an $\Omega(n^{1/d})$ regime for general bounded-support
distributions.
More recently, Das et al.~\citep{das2025bregman}~generalize universal-law-type results beyond squared loss to
a broad family of Bregman divergence losses (covering common classification losses), and related work also studies robustness laws under weight-bounded classes
\citep{husain2021weight}.

\subsection{Robust generalization and robust overfitting}
A persistent challenge in adversarial robustness is robust generalization, where robust training error can be low while test error remains high.
It is known that robust learning can require significantly more labeled data than standard learning, even in simple data models 
\citep{schmidt2018moredata,carmon2019unlabeled,zhai2019unlabeled}.
It has also been shown by \citep{yin2019rademacher} that robust generalization bounds exhibit unavoidable dimension dependence unless additional structure is imposed. A closely related phenomenon is \emph{robust overfitting}: continued adversarial training can improve
robust training loss while degrading robust test performance
\citep{rice2020overfitting}. 
Subsequent work proposes mechanisms to mitigate robust overfitting via learned smoothing
regularization \citep{chen2021smoothening} or by analyzing which subsets of adversarial examples drive
overfitting and modifying training to counteract it \citep{yu2022robustoverfitting}. Moreover, it has recently been shown why why robust generalization is intrinsically challenging from an expressive
power perspective \citep{li2022why}. Finally, several recent papers develop stability-based and training-dynamics-based explanations for
robust generalization behavior \citep{zhang2024stability_shallow,cheng2024free,zhang2024lesscertain}.
Additionally, recent work further formalizes the clean-generalization vs. robust-overfitting dichotomy through
representation complexity and phase-transition analyses \citep{li2025cgro}.

\section{Law of Robustness and Robust Generalization}
\label{sec:gap-to-complexity}
\subsection{Notation}
All random variables are defined on a common probability space. The data-generating distribution is denoted by \(P\), and a generic sample is written as $Z=(X,Y)\sim P$, where \((X,Y)\) takes values in \(\mathcal X\times\mathcal Y\subseteq \mathbb R^d\times[-1,1]\). Let $\sigma^2=\mathbb E[\operatorname{Var}(Y\mid X)]$
be the irreducible noise level. The training sample is written as $S=\{Z_i=(X_i,Y_i)\}_{i=1}^n\sim P^n$.
 \(P_n\) denotes the empirical measure associated with \(S\). Thus, for any measurable function \(g:\mathcal X\times\mathcal Y\to\mathbb R\), we write
\[
Pg=\mathbb E[g(X,Y)]
\qquad\text{and}\qquad
P_n g=\frac1n\sum_{i=1}^n g(X_i,Y_i).
\]

For \(\rho\ge 0\), let \(B_\rho(x)=\) denote the closed Euclidean ball of radius \(\rho\) centered at \(x\). We also define the \(\rho\)-enlargement of the input space by
\(
\mathcal X_\rho=\{x+u:x\in\mathcal X,\ \|u\|_2\le \rho\}.
\) For \(L\ge 0\), let
\[
\mathcal F_L
=
\left\{
f:\mathcal X_\rho\to[-1,1]:
|f(x)-f(x')|\le L\|x-x'\|_2
\ \text{for all }x,x'\in\mathcal X_\rho
\right\}
\]
be the class of bounded \(L\)-Lipschitz functions. When the perturbation radius is clear from context, we suppress the dependence of \(\mathcal F_L\) on \(\mathcal X_\rho\).

For \(f\in\mathcal F_L\), the clean squared loss and the adversarial squared loss at radius \(\rho\) are $\ell_{0,f}(x,y)=(f(x)-y)^2$ and $\ell_{\rho,f}(x,y)=\sup_{\tilde x\in B_\rho(x)}(f(\tilde x)-y)^2$. When \(\rho=0\), this reduces to the clean loss. The corresponding clean population and empirical risks are

$R_0(f)=P\ell_{0,f}$ and $\widehat R_0(f)=P_n\ell_{0,f}$,
while the robust population risks and empirical risks are
$R_\rho(f)=P\ell_{\rho,f}
\qquad\text{and}\qquad
\widehat R_\rho(f)=P_n\ell_{\rho,f}$.
The robust generalization gap is denoted by
$\operatorname{Gap}_\rho(f)=R_\rho(f)-\widehat R_\rho(f)$.

The clean and robust loss classes induced by \(\mathcal F_L\) are
\[
\mathcal L_0(\mathcal F_L)=\{\ell_{0,f}:f\in\mathcal F_L\},
\qquad
\mathcal L_\rho(\mathcal F_L)=\{\ell_{\rho,f}:f\in\mathcal F_L\}.
\]
When the underlying predictor class is clear, we abbreviate these by \(\mathcal L_0\) and \(\mathcal L_\rho\).

Finally, when \((T,d_T)\) is a metric space and \(\eta>0\), the covering number
$
\mathcal N(\eta,T,d_T)
$
is the smallest integer \(N\) for which there exist
\(t_1,\ldots,t_N\in T\) such that
\(
T\subseteq \bigcup_{j=1}^N
\{t\in T:d_T(t,t_j)\le\eta\}
\). If no such finite cover exists, we set
\(\mathcal N(\eta,T,d_T)=\infty\).  

For a class \(\mathcal G\) of functions
and a probability measure \(Q\), we write $\|g-h\|_{L_2(Q)}
=
\left(Q(g-h)^2\right)^{1/2}$.
Thus \(\mathcal N(\eta,\mathcal G,L_2(Q))\) denotes the covering number of
\(\mathcal G\) under the \(L_2(Q)\)-metric.

We first show that clean overfitting creates a robust train--test gap. Furthermore, we use Rademacher generalization bounds to obtain a lower bound on the complexity of the robust loss class.

\begin{lemma}[Robust generalization gap and overfitting]
\label{thm:robust-noise-floor-gap}
For any \(f\in\mathcal F_L\),
$
\widehat R_\rho(f)
\le
\left(\sqrt{\widehat R_0(f)}+L\rho\right)^2,
$
and hence
$\operatorname{Gap}_\rho(f)
\ge
\sigma^2-\left(\sqrt{\widehat R_0(f)}+L\rho\right)^2$.
\end{lemma}

\begin{proof}
For each \((X_i,Y_i)\in S\) and every \(\tilde X_i\in B_\rho(X_i)\),
$|f(\tilde X_i)-Y_i|
\le
|f(X_i)-Y_i|+|f(\tilde X_i)-f(X_i)|
\le
|f(X_i)-Y_i|+L\rho$.
Taking the supremum over \(\tilde X_i\in B_\rho(X_i)\), squaring, and averaging gives
$\widehat R_\rho(f)
\le
\frac1n\sum_{i=1}^n
\left(|f(X_i)-Y_i|+L\rho\right)^2$. By Cauchy--Schwarz,
$
\frac1n\sum_{i=1}^n |f(X_i)-Y_i|
\le
\sqrt{\widehat R_0(f)}.
$
Therefore
\[
\widehat R_\rho(f)
\le
\widehat R_0(f)+2L\rho\sqrt{\widehat R_0(f)}+L^2\rho^2
=
\left(\sqrt{\widehat R_0(f)}+L\rho\right)^2.
\]
Finally, \(R_\rho(f)\ge R_0(f)\ge\sigma^2\), where the last inequality follows from the squared-loss variance decomposition.  This proves the claim.
\end{proof}

\begin{corollary}[Overfitting implies a robust gap]
\label{cor:below-noise-robust-gap}
If \(0<\varepsilon<\sigma^2\) and
\(
\widehat R_0(f)\le\sigma^2-\varepsilon,
\)
then
\[
\operatorname{Gap}_\rho(f)
\ge
\Gamma_\rho(\varepsilon,L)
\]
where $\Gamma_{\rho}(\epsilon, L):=
\sigma^2-\left(\sqrt{\sigma^2-\varepsilon}+L\rho\right)^2
=
\varepsilon
-
2L\rho\sqrt{\sigma^2-\varepsilon}
-
L^2\rho^2$.
In particular, the lower bound is positive whenever $L\rho<\sigma-\sqrt{\sigma^2-\varepsilon}$.
\end{corollary}

\begin{proof}
The assumption gives $\sqrt{\widehat R_0(f)} \le \sqrt{\sigma^2-\varepsilon}$; substituting in Theorem~\ref{thm:robust-noise-floor-gap} proves the result.
\end{proof}

We now convert the robust gap into a lower bound on the complexity of the robust loss class.  Since \(Y\in[-1,1]\), every robust squared loss satisfies
$
0\le \ell_{\rho,f}\le 4.
$

\begin{definition}[Rademacher complexities]
Let \(S=\{Z_i\}_{i=1}^n\) be a sample, and let \(\xi_1,\dots,\xi_n\) be independent Rademacher random variables, independent of \(S\), with
\(
\mathbb P(\xi_i=1)=\mathbb P(\xi_i=-1)=\tfrac12.
\)

For a class \(\mathcal G\) of real-valued functions on \(\mathcal X\times\mathcal Y\), the empirical and expected Rademacher complexity is
\[
\widehat{\mathfrak R}_S(\mathcal G)
=
\mathbb E_\xi
\left[
\sup_{g\in\mathcal G}
\frac1n\sum_{i=1}^n \xi_i g(Z_i)
\right]
\qquad \text{and} \qquad 
\mathfrak R_n(\mathcal G)
=
\mathbb E_S\widehat{\mathfrak R}_S(\mathcal G).
\]
\end{definition}

\begin{lemma}[Complexity forced by a robust gap]
\label{thm:rademacher-lower-bound} With probability at least \(1-\delta\) over \(S\), simultaneously for all \(f\in\mathcal F_L\),
\[
\mathfrak R_n(\mathcal L_\rho)
\ge
\frac12
\left[
\sigma^2-\left(\sqrt{\widehat R_0(f)}+L\rho\right)^2
-
4\sqrt{\frac{\log(1/\delta)}{2n}}
\right].
\]
Consequently, if \(\widehat R_0(f)\le\sigma^2-\varepsilon\), then
\[
\mathfrak R_n(\mathcal L_\rho)
\ge
\frac12
\left[
\Gamma_\rho(\varepsilon,L)
-
4\sqrt{\frac{\log(1/\delta)}{2n}}
\right].
\]
\end{lemma}

\begin{proof}
By the standard Rademacher generalization bound for classes bounded in \([0,4]\), with probability at least \(1-\delta\), every \(g\in\mathcal L_\rho\) satisfies
\[
Pg-P_n g
\le
2\mathfrak R_n(\mathcal L_\rho)
+
4\sqrt{\frac{\log(1/\delta)}{2n}}.
\]
Applying this to \(g=\ell_{\rho,f}\) gives, uniformly over \(f\in\mathcal F_L\),
\[
\operatorname{Gap}_\rho(f)
\le
2\mathfrak R_n(\mathcal L_\rho)
+
4\sqrt{\frac{\log(1/\delta)}{2n}}.
\]
Combining this upper bound with Lemma~\ref{thm:robust-noise-floor-gap} yields
\[
\sigma^2-\left(\sqrt{\widehat R_0(f)}+L\rho\right)^2
\le
2\mathfrak R_n(\mathcal L_\rho)
+
4\sqrt{\frac{\log(1/\delta)}{2n}}.
\]
Rearranging proves the first inequality.  The below-noise version follows from Corollary~\ref{cor:below-noise-robust-gap}.
\end{proof}

%

We now compare robust and
clean squared-loss classes.  

\begin{lemma}[Extremal envelopes]
\label{lem:extremal-envelopes}
For \(f\in\mathcal F_L\), define
$f_\rho^+(x)=\sup_{\tilde x\in B_\rho(x)} f(\tilde x)$ and
$f_\rho^-(x)=\inf_{\tilde x\in B_\rho(x)} f(\tilde x)$.
Both \(f_\rho^+\) and \(f_\rho^-\) are bounded in \([-1,1]\) and
\(L\)-Lipschitz on \(\mathcal X\).  Moreover, for every
\((x,y)\in\mathcal X\times[-1,1]\),
$\ell_{\rho,f}(x,y)
= \max\left\{
(f_\rho^+(x)-y)^2,\,
(f_\rho^-(x)-y)^2
\right\}$. See~ Appendix~\ref{sec:appendix}-\ref{lem:extremal-envelopes-appendix} for the proof.
\end{lemma}

\begin{corollary}[Squared-loss contraction]
\label{cor:squared-loss-contraction}
Let \(y_1,\ldots,y_n\in[-1,1]\), and let \(A\subseteq[-1,1]^n\).  Then
\[
\widehat{\mathfrak R}
\left(
\left\{
\bigl((a_i-y_i)^2\bigr)_{i=1}^n:a\in A
\right\}
\right)
\le
4\,\widehat{\mathfrak R}(A).
\]
\end{corollary}

\begin{proof}
Set \(\phi_i(t)=(t-y_i)^2-y_i^2\).  Then \(\phi_i(0)=0\), and
\(\phi_i\) is \(4\)-Lipschitz on \([-1,1]\).  Extend \(\phi_i\) to a
\(4\)-Lipschitz function on \(\mathbb R\). We use the local contraction Lemma~\ref{lem:local-contraction} (Appendix \ref{sec:appendix})
gives
\[
\widehat{\mathfrak R}
\left(
\left\{
\bigl((a_i-y_i)^2-y_i^2\bigr)_{i=1}^n:a\in A
\right\}
\right)
\le
4\,\widehat{\mathfrak R}(A).
\]
Adding the fixed vector \((y_i^2)_{i=1}^n\) does not change Rademacher
complexity.
\end{proof}

\begin{lemma}[Coordinatewise maxima]
\label{lem:coordinatewise-maxima}
Let \(A,B\subseteq\mathbb R^n\), and define
$
A\vee B
=
\{(a_1\vee b_1,\ldots,a_n\vee b_n):a\in A,\ b\in B\}$.
Then
$
\widehat{\mathfrak R}(A\vee B)
\le
\widehat{\mathfrak R}(A)+\widehat{\mathfrak R}(B)$.
\end{lemma}

\begin{proof}
Using \(u\vee v=(u+v+|u-v|)/2\),
\[
\widehat{\mathfrak R}(A\vee B)
\le
\frac12\widehat{\mathfrak R}(A)
+
\frac12\widehat{\mathfrak R}(B)
+
\frac12
\mathbb E_\xi
\left[
\sup_{a\in A,b\in B}
\frac1n\sum_{i=1}^n \xi_i |a_i-b_i|
\right].
\]
Apply Lemma~\ref{lem:local-contraction} to the class
\(\{a-b:a\in A,\ b\in B\}\) and to \(t\mapsto |t|\).  The last expectation is
at most
\[
\mathbb E_\xi
\left[
\sup_{a\in A,b\in B}
\frac1n\sum_{i=1}^n \xi_i(a_i-b_i)
\right]
\le
\widehat{\mathfrak R}(A)+\widehat{\mathfrak R}(-B).
\]
Since \(\widehat{\mathfrak R}(-B)=\widehat{\mathfrak R}(B)\), the claim
follows.
\end{proof}

\begin{theorem}[Local robust-to-clean reduction]
\label{prop:local-robust-to-clean-reduction}
Fix a sample \(S=\{(X_i,Y_i)\}_{i=1}^n\).  Let
\[
\mathcal F_L(\mathcal X)
=
\{h:\mathcal X\to[-1,1]: |h(x)-h(x')|\le L\|x-x'\|_2
\text{ for all }x,x'\in\mathcal X\}.
\]
Then
$\widehat{\mathfrak R}_S\bigl(\mathcal L_\rho(\mathcal F_L)\bigr)
\le
2\,\widehat{\mathfrak R}_S\bigl(\mathcal L_0(\mathcal F_L(\mathcal X))\bigr)
\le
8\,\widehat{\mathfrak R}_{X_1^n}\bigl(\mathcal F_L(\mathcal X)\bigr)$.
\end{theorem}

\begin{proof}
By Lemma~\ref{lem:extremal-envelopes}, every robust-loss vector $\bigl(\ell_{\rho,f}(X_i,Y_i)\bigr)_{i=1}^n$
is the coordinatewise maximum of two clean squared-loss vectors induced by
\(f_\rho^+\) and \(f_\rho^-\).  Since both envelopes belong to
\(\mathcal F_L(\mathcal X)\), the set of robust-loss vectors is contained in
\(A\vee A\), where
\[
A
=
\left\{
\bigl((h(X_i)-Y_i)^2\bigr)_{i=1}^n:
h\in\mathcal F_L(\mathcal X)
\right\}.
\]
By monotonicity of Rademacher complexity and
Lemma~\ref{lem:coordinatewise-maxima},
\[
\widehat{\mathfrak R}_S\bigl(\mathcal L_\rho(\mathcal F_L)\bigr)
\le
\widehat{\mathfrak R}(A\vee A)
\le
2\,\widehat{\mathfrak R}(A)
=
2\,\widehat{\mathfrak R}_S\bigl(\mathcal L_0(\mathcal F_L(\mathcal X))\bigr).
\]
The second inequality follows from Corollary~\ref{cor:squared-loss-contraction},
because \(h(X_i)\in[-1,1]\) and \(Y_i\in[-1,1]\):
\[
\widehat{\mathfrak R}_S\bigl(\mathcal L_0(\mathcal F_L(\mathcal X))\bigr)
\le
4\,\widehat{\mathfrak R}_{X_1^n}\bigl(\mathcal F_L(\mathcal X)\bigr).
\]
Combining the two bounds proves the claim. Using Lemma 3.8 in \citep{wu2023law}, L = \(\Omega(n^{1/d})\) 
\end{proof}
\section{Local Rademacher complexity and Robust Generalization}
\label{sec:localized-rademacher}

We now localize the robust loss class by its true \(L_2(P)\)-radius.  For
\(r\ge0\), define $\mathcal L_\rho(r) = \left\{
\ell_{\rho,f}: f\in\mathcal F_L,\;
P\ell_{\rho,f}^2\le r
\right\}$.

Throughout this section, let $\X\subset \R^d$ be nonempty, compact, and connected, with
$D:=\diam(\X)<\infty$. Because the robust loss evaluates \(f\) on the enlarged set \(\mathcal X_\rho\), while samples lie in \(\mathcal X\), we first record a simple extension lemma. It shows that any Lipschitz function defined on \(\mathcal X\) can be extended to \(\mathcal X_\rho\) without increasing its Lipschitz constant or leaving the range \([-1,1]\). This allows us to treat covers of restrictions to \(\mathcal X\) as covers induced by functions in the original class on \(\mathcal X_\rho\).

For purposes of bounding the empirical metric entropy on samples from \(\mathcal X\), it is enough to construct covers of Lipschitz functions on \(\mathcal X\). Any such approximant may be extended back to \(\mathcal X_\rho\) while remaining in \(\mathcal B_L\). This can be proved using McShane's Theorem (Appendix \ref{sec:appendix}).We next prove a covering bound for bounded Lipschitz functions on \(\mathcal X\).

\begin{theorem}
[Metric entropy of bounded Lipschitz predictors]
\label{lem:lipschitz-predictor-entropy}Let $\mathcal F_L(\mathcal X)$ be as defined above. There is a constant \(c_d>0\), depending only on \(d\), such that for every
\(0<\eta\le1\),
\[
\log \mathcal N
\left(
\eta,\mathcal F_L(\mathcal X),\|\cdot\|_\infty
\right)
\le
c_d
\left(
\frac{1+LD}{\eta}
\right)^d .
\]
\end{theorem}

\begin{proof}
The case \(L=0\) is immediate.  Assume \(L>0\) and set
\(\delta=\eta/(8L)\).  Let \(x_1,\ldots,x_m\) be a maximal
\(\delta\)-separated subset of \(\mathcal X\).  Then it is a \(\delta\)-net,
and the standard packing bound gives
\[
m\le \left(1+\frac{2D}{\delta}\right)^d
=
\left(1+\frac{16LD}{\eta}\right)^d .
\]

Connect \(i\) and \(j\) if
\(\overline B(x_i,\delta)\cap\overline B(x_j,\delta)\neq\varnothing\).
This graph is connected; otherwise the corresponding unions of
\(\delta\)-balls over two graph components would separate the connected set
\(\mathcal X\).  Fix a spanning tree rooted at an arbitrary vertex.

Let \(h=\eta/4\) and $G_\eta=h\mathbb Z\cap[-1-h/2,1+h/2]$.
Choose \(\pi_\eta(t)\in G_\eta\) with
$|t-\pi_\eta(t)|\le h/2=\eta/8$, $t\in[-1,1]$.
Then \(|G_\eta|\le 2+8/\eta\).  For \(f\in\mathcal F_L(\mathcal X)\), define
its quantized net-value vector by
\[
q(f)=(q_1(f),\ldots,q_m(f)),
\qquad
q_j(f)=\pi_\eta(f(x_j)).
\]
The root coordinate has at most \(2+8/\eta\) choices.  If \(p(j)\) is the
parent of \(j\) in the spanning tree, then
\(\|x_j-x_{p(j)}\|_2\le 2\delta\), and hence $|f(x_j)-f(x_{p(j)})|\le 2L\delta=\eta/4$.
Therefore $|q_j(f)-q_{p(j)}(f)| \le \eta/8+\eta/4+\eta/8 = \eta/2$.
Since the grid spacing is \(\eta/4\), each non-root coordinate has at most five
possible values once its parent value is fixed.  Thus the number of admissible
quantized net-value vectors is at most $\left(2+\frac{8}{\eta}\right)5^{m-1}$.

If \(f,g\in\mathcal F_L(\mathcal X)\) have the same quantized net-value vector,
then for any \(x\in\mathcal X\), choosing \(x_j\) with
\(\|x-x_j\|_2\le\delta\) gives
$ |f(x)-g(x)| \le L\delta+\eta/8+\eta/8+L\delta = \eta/2$.
Hence each nonempty fiber of the quantization map has
\(\|\cdot\|_\infty\)-diameter at most \(\eta\).  Taking one representative from
each nonempty fiber yields an \(\eta\)-cover, so
\[
\log \mathcal N
\left(
\eta,\mathcal F_L(\mathcal X),\|\cdot\|_\infty
\right)
\le
\log\left(2+\frac{8}{\eta}\right)
+
(\log 5)\left(1+\frac{16LD}{\eta}\right)^d .
\]
Let \(A=(1+LD)/\eta\).  Since \(0<\eta\le1\), \(A\ge1\), and
\[
\log\left(2+\frac{8}{\eta}\right)
\le
\frac{10}{\eta}
\le
10A^d .
\]
Also, $1+\frac{16LD}{\eta} \le 17A$
Hence
$
\log \mathcal N
\left(
\eta,\mathcal F_L(\mathcal X),\|\cdot\|_\infty
\right)
\le
\bigl(10+17^d\log 5\bigr)A^d$.
Absorbing the constant into \(c_d\) gives
\[
\log \mathcal N
\left(
\eta,\mathcal F_L(\mathcal X),\|\cdot\|_\infty
\right)
\le
c_d
\left(\frac{1+LD}{\eta}\right)^d .
\]
\end{proof}


\begin{lemma}[Metric entropy of robust squared losses]
\label{lem:robust-loss-entropy}
There is a constant \(C_d>0\), depending only on \(d\), such that for every
\(0<\eta\le4\),
\[
\log \mathcal N
\left(
\eta,\mathcal L_\rho(\mathcal F_L),\|\cdot\|_\infty
\right)
\le
C_d
\left(
\frac{1+LD}{\eta}
\right)^d .
\]
\end{lemma}

See Appendix~\ref{sec:appendix}-\ref{lem:robust-loss-entropy-appendix} for the proofs.

\begin{lemma}[Polynomial entropy implies a local bound]
\label{lem:polynomial-entropy-local-rademacher}
Let \(\mathcal G\) be a measurable class of functions bounded in \([0,b]\),
where \(b>0\). Assume that for some \(p>2\) and \(A>0\),
\[
\sup_Q
\log \mathcal N
\left(
\eta,\mathcal G,L_2(Q)
\right)
\le
A\eta^{-p}
\qquad
\text{for all }0<\eta\le b,
\]
where the supremum is over all finitely supported probability measures \(Q\).
For \(r\ge0\), define
\[
\mathcal G(r)
=
\{g\in\mathcal G:Pg^2\le r\}.
\]
Let $\bar A:=1\vee A$.
Then there is a constant \(C_{p,b}>0\), depending only on \(p\) and \(b\), such that
\[
\mathfrak R_n(\mathcal G(r))
\le
C_{p,b}
\left(
\bar A^{1/p}n^{-1/p}
+
\sqrt{\frac rn}
\right).
\]
In particular, if \(A\ge1\), then
\[
\mathfrak R_n(\mathcal G(r))
\le
C_{p,b}
\left(
A^{1/p}n^{-1/p}
+
\sqrt{\frac rn}
\right).
\]
\end{lemma}

See Appendix~\ref{sec:appendix}-\ref{lem:polynomial-entropy-local-rademacher-appendix} for the proofs.

\begin{theorem}[Localized robust Rademacher upper bound]
\label{thm:localized-robust-rademacher-upper}
Assume \(d\ge3\).  There is a constant \(C_d>0\), depending only on \(d\),
such that for every \(r\in[0,16]\),
\[
\mathfrak R_n\bigl(\mathcal L_\rho(r)\bigr)
\le
C_d
\left(
(1+LD)n^{-1/d}
+
\sqrt{\frac rn}
\right).
\]
\end{theorem}

\begin{proof}
Since \(f:\mathcal X_\rho\to[-1,1]\) and \(Y\in[-1,1]\), every robust squared
loss satisfies
\[
0\le \ell_{\rho,f}\le4.
\]
By Lemma~\ref{lem:robust-loss-entropy}, and since
\(\|\cdot\|_{L_2(Q)}\le\|\cdot\|_\infty\),
\[
\sup_Q
\log \mathcal N
\left(
\eta,\mathcal L_\rho(\mathcal F_L),L_2(Q)
\right)
\le
C_d
\left(
\frac{1+LD}{\eta}
\right)^d .
\]
Apply Lemma~\ref{lem:polynomial-entropy-local-rademacher} with
\(\mathcal G=\mathcal L_\rho(\mathcal F_L)\), \(p=d\), \(b=4\), and
\(A=C_d(1+LD)^d\).  This gives
\[
\mathfrak R_n\bigl(\mathcal L_\rho(r)\bigr)
\le
C_d
\left(
(1+LD)n^{-1/d}
+
\sqrt{\frac rn}
\right),
\]
after absorbing constants depending only on \(d\).
\end{proof}

\begin{lemma}[Localized robust gap forces localized complexity]
\label{lem:localized-robust-gap-forces-complexity}
Fix \(r\in[0,16]\) and \(\delta\in(0,1)\).  With probability at least
\(1-\delta\) over \(S\), every \(f\in\mathcal F_L\) satisfying
\[
P\ell_{\rho,f}^2\le r
\qquad\text{and}\qquad
\widehat R_0(f)\le \sigma^2-\varepsilon
\]
also satisfies
\[
\mathfrak R_n\bigl(\mathcal L_\rho(r)\bigr)
\ge
\frac12
\left[
\Gamma_\rho(\varepsilon,L)
-
4\sqrt{\frac{\log(1/\delta)}{2n}}
\right].
\]
\end{lemma}
This can be proved using standard McDiarmid's inequality. The detailed proof is in Appendix \ref{sec:appendix}.

\begin{corollary}[Localized compatibility condition]
\label{cor:localized-compatibility}
Under the assumptions of Theorem~\ref{thm:localized-robust-rademacher-upper},
fix \(r\in[0,16]\) and \(\delta\in(0,1)\).  With probability at least
\(1-\delta\), every \(f\in\mathcal F_L\) satisfying
\[
P\ell_{\rho,f}^2\le r
\qquad\text{and}\qquad
\widehat R_0(f)\le\sigma^2-\varepsilon
\]
must obey
\[
\Gamma_\rho(\varepsilon,L)
\le
2C_d
\left(
(1+LD)n^{-1/d}
+
\sqrt{\frac rn}
\right)
+
4\sqrt{\frac{\log(1/\delta)}{2n}}.
\]
Equivalently,
\[
\varepsilon
-
2L\rho\sqrt{\sigma^2-\varepsilon}
-
L^2\rho^2
\le
2C_d
\left(
(1+LD)n^{-1/d}
+
\sqrt{\frac rn}
\right)
+
4\sqrt{\frac{\log(1/\delta)}{2n}}.
\]
\end{corollary}

\begin{proof}
Combine Corollary~\ref{cor:localized-robust-gap-forces-complexity} with
Theorem~\ref{thm:localized-robust-rademacher-upper}.
\end{proof}

\section{Local Lipschitz Scaling from Localized Robust Rademacher Bounds}
\label{sec:local-lipschitz-scaling}

We now rewrite the localized compatibility condition as a lower bound on the
Lipschitz scale \(L\).  Throughout this section, assume \(d\ge3\),
\(D=\diam(\mathcal X)<\infty\), \(0<\varepsilon<\sigma^2\),
\(r\in[0,16]\), \(\delta\in(0,1)\), and \(\rho\ge0\).  Define $s_\varepsilon:=\sqrt{\sigma^2-\varepsilon}$.
By Corollary~\ref{cor:localized-compatibility}, with probability at least
\(1-\delta\), every \(f\in\mathcal F_L\) satisfying $P\ell_{\rho,f}^2\le r$ and $\widehat R_0(f)\le\sigma^2-\varepsilon$
must obey
\begin{equation}
\label{eq:localized-compatibility-expanded}
\varepsilon
-
2L\rho s_\varepsilon
-
L^2\rho^2
\le
2C_d
\left(
(1+LD)n^{-1/d}
+
\sqrt{\frac rn}
\right)
+
4\sqrt{\frac{\log(1/\delta)}{2n}} .
\end{equation}

Define the localized excess margin
\begin{equation}
\label{eq:local-gap}
\Delta_{n,r,\delta}
:=
\varepsilon
-
2C_d n^{-1/d}
-
2C_d\sqrt{\frac rn}
-
4\sqrt{\frac{\log(1/\delta)}{2n}} .
\end{equation}
Then \eqref{eq:localized-compatibility-expanded} implies
\begin{equation}
\label{eq:quadratic-compatibility}
\Delta_{n,r,\delta}
\le
\rho^2L^2
+
\left(
2\rho s_\varepsilon
+
2C_dD n^{-1/d}
\right)L .
\end{equation}
Thus a nontrivial lower bound on \(L\) is obtained precisely in the regime
\[
\Delta_{n,r,\delta}>0.
\]

\begin{corollary}[Local Lipschitz scaling]
\label{cor:local-lipschitz-scaling}
On the event of probability at least \(1-\delta\) from
Corollary~\ref{cor:localized-compatibility}, suppose that $\Delta_{n,r,\delta}>0$.
Then every \(f\in\mathcal F_L\) satisfying $P\ell_{\rho,f}^2\le r$ and $\widehat R_0(f)\le\sigma^2-\varepsilon$
must satisfy the following lower bounds.

If \(\rho>0\), then
\begin{equation}
\label{eq:exact-local-lipschitz}
L
\ge
\frac{
-
B_{n,\rho,\varepsilon}
+
\sqrt{
B_{n,\rho,\varepsilon}^2
+
4\rho^2\Delta_{n,r,\delta}
}
}{
2\rho^2
},
\end{equation}
where
\begin{equation}
\label{eq:B-local}
B_{n,\rho,\varepsilon}
:=
2\rho\sqrt{\sigma^2-\varepsilon}
+
2C_dD n^{-1/d}.
\end{equation}
Equivalently,
\begin{equation}
\label{eq:exact-local-lipschitz-rationalized}
L
\ge
\frac{
2\Delta_{n,r,\delta}
}{
B_{n,\rho,\varepsilon}
+
\sqrt{
B_{n,\rho,\varepsilon}^2
+
4\rho^2\Delta_{n,r,\delta}
}
}.
\end{equation}
In particular,
\begin{equation}
\label{eq:local-lipschitz-simple-lower}
L
\ge
\frac{
\Delta_{n,r,\delta}
}{
2\rho\sqrt{\sigma^2-\varepsilon}
+
2C_dD n^{-1/d}
+
\rho\sqrt{\Delta_{n,r,\delta}}
}.
\end{equation}
Consequently, up to constants depending only on \(d\),
\begin{equation}
\label{eq:local-lipschitz-order}
L
=
\Omega\left(
\frac{
\Delta_{n,r,\delta}
}{
\rho\sqrt{\sigma^2-\varepsilon}
+
D n^{-1/d}
+
\rho\sqrt{\Delta_{n,r,\delta}}
}
\right).
\end{equation}
\end{corollary}

See Appendix~\ref{sec:appendix}-\ref{cor:local-lipschitz-scaling-appendix} for the proofs.

\begin{remark}[A weaker form using \(\sigma\)]
Since
$\sqrt{\sigma^2-\varepsilon}\le\sigma$,
one may replace \(B_{n,\rho,\varepsilon}\) by the larger quantity
$\overline B_{n,\rho}
:=
2\rho\sigma
+
2C_dD n^{-1/d}$.
This gives the weaker but sometimes simpler valid lower bound, for
\(\rho>0\),
\[
L
\ge
\frac{
-
\overline B_{n,\rho}
+
\sqrt{
\overline B_{n,\rho}^2
+
4\rho^2\Delta_{n,r,\delta}
}
}{
2\rho^2
},
\]
and hence
\[
L
=
\Omega\left(
\frac{
\Delta_{n,r,\delta}
}{
\rho\sigma
+
D n^{-1/d}
+
\rho\sqrt{\Delta_{n,r,\delta}}
}
\right),
\]
whenever \(\Delta_{n,r,\delta}>0\).  This is a relaxation of the exact bound,
not the exact compatibility condition.
\end{remark}

Ignoring only constants depending on \(d\), the confidence term, and the
displayed lower-order localization terms, the exact nonvacuous regime can be
summarized as
\[
L
=
\Omega\left(
\frac{
\varepsilon
-
n^{-1/d}
-
\sqrt{r/n}
}{
\rho\sqrt{\sigma^2-\varepsilon}
+
D n^{-1/d}
+
\rho\sqrt{
\varepsilon
-
n^{-1/d}
-
\sqrt{r/n}
}
}
\right),
\]
provided the numerator is positive.  


\paragraph{Dependence on the number of parameters.}
The localized Rademacher argument above is geometric: it applies to the
bounded Lipschitz class on a \(d\)-dimensional domain and therefore depends
on \(n,d,D,r,\rho,\sigma,\varepsilon\), and \(\delta\), but not directly on
the number of network parameters \(p\).  A parameter-count dependence enters
only after imposing an additional assumptions on the distribuion like the Isoperimteric distribution \citep{bubeck2021universal}.

\section{Discussion}
This work takes a step towards unifying worst-case robustness and robust generalization, as posed by \citep{bubeck2021universal}. In the global setting, our analysis connects robust generalization gaps to the global Rademacher complexity of the induced robust loss class, finding that the Lipschitz constant remains of the order $\Omega(n^{1/d})$. Interestingly, when using local notions of Rademacher complexity, we find a contrary outcome: the order of the Lipschitz constant changes with $\rho$ and the concentration term $\sqrt {r/n}$.

We surmise that this is because, in the global setting, the number of functions is huge: in this massive ocean of functions, perturbations are imperceptible in terms of order. However, when we zoom in and look at the functions which might actually be reached by the optimizing algorithm, the changes suddenly become visible: the order of the Lipschitz constant is now directly affected!

There are several natural directions for future work. First, there needs to be experimental exploration of the significance of this concentration term. Next, our theoretical results are stated for square-loss, whereas Robust Generalization is usually in the classification setting using cross-entropy loss; extending the analysis to Bregmann divergence losses would close this gap.  Further, distributional robustness may provide another route to connect robust generalization with worst-case robustness. 

\bibliographystyle{plainnat}
\bibliography{references}

\appendix

\section{Technical appendices and supplementary material}
\label{sec:appendix}

\begin{lemma}[Extremal envelopes]
\label{lem:extremal-envelopes-appendix}
For \(f\in\mathcal F_L\), define
\[
f_\rho^+(x)=\sup_{\tilde x\in B_\rho(x)} f(\tilde x),
\qquad
f_\rho^-(x)=\inf_{\tilde x\in B_\rho(x)} f(\tilde x).
\]
Both \(f_\rho^+\) and \(f_\rho^-\) are bounded in \([-1,1]\) and
\(L\)-Lipschitz on \(\mathcal X\).  Moreover, for every
\((x,y)\in\mathcal X\times[-1,1]\),
\[
\ell_{\rho,f}(x,y)
=
\max\left\{
(f_\rho^+(x)-y)^2,\,
(f_\rho^-(x)-y)^2
\right\}.
\]
\end{lemma}

\begin{proof}
Since \(B_\rho(x)\subseteq\mathcal X_\rho\), the envelopes are well-defined
and bounded in \([-1,1]\).  For \(x,x'\in\mathcal X\) and \(\|u\|_2\le\rho\),
\[
f(x+u)\le f(x'+u)+L\|x-x'\|_2.
\]
Taking the supremum over \(u\) gives
\[
f_\rho^+(x)\le f_\rho^+(x')+L\|x-x'\|_2.
\]
Exchanging \(x,x'\) proves that \(f_\rho^+\) is \(L\)-Lipschitz.  Since
\(f_\rho^-=-(-f)_\rho^+\), the same holds for \(f_\rho^-\).

Fix \((x,y)\).  Let
\[
A_x=\{f(\tilde x):\tilde x\in B_\rho(x)\},\qquad
M=\sup A_x,\quad m=\inf A_x.
\]
Since \(A_x\subseteq[m,M]\),
\[
\sup_{z\in A_x}(z-y)^2
\le
\max\{(M-y)^2,(m-y)^2\}.
\]
The reverse inequality follows by approximating \(M\) and \(m\) with
sequences in \(A_x\).  Hence
\[
\sup_{z\in A_x}(z-y)^2
=
\max\{(M-y)^2,(m-y)^2\},
\]
which gives the claimed identity.
\end{proof}

\begin{lemma}[Local contraction]
\label{lem:local-contraction}
Let \(A\subseteq\mathbb R^n\).  For each \(i\), let
\(\phi_i:\mathbb R\to\mathbb R\) be \(C\)-Lipschitz and satisfy
\(\phi_i(0)=0\).  Then
\[
\widehat{\mathfrak R}
\left(
\{(\phi_1(a_1),\ldots,\phi_n(a_n)):a\in A\}
\right)
\le
C\,\widehat{\mathfrak R}(A).
\]
\end{lemma}

\begin{lemma}[Range-preserving Lipschitz extension]
\label{lem:extension}
Let $g:\X\to[-1,1]$ be $L$-Lipschitz. Then there exists an $L$-Lipschitz function $\widetilde g:\X_\rho\to[-1,1]$ such that $\widetilde g=g$ on $\X$.
\end{lemma}
\begin{proof}
By the McShane extension theorem, the function
\[
\widetilde g_0(z):=\inf_{x\in\X}\{g(x)+L\norm{z-x}\},\qquad z\in\X_\rho,
\]
is an $L$-Lipschitz extension of $g$ from $\X$ to $\X_\rho$. The metric projection $\Pi_{[-1,1]}(t):=\max\{-1,\min\{t,1\}\}$ is $1$-Lipschitz on $\R$, so $\widetilde g:=\Pi_{[-1,1]}\circ \widetilde g_0$ is still $L$-Lipschitz, takes values in $[-1,1]$, and agrees with $g$ on $\X$ because $g(\X)\subset[-1,1]$.
\end{proof}

\begin{lemma}[Localized robust gap forces localized complexity]
\label{cor:localized-robust-gap-forces-complexity}
Fix \(r\in[0,16]\) and \(\delta\in(0,1)\).  With probability at least
\(1-\delta\) over \(S\), every \(f\in\mathcal F_L\) satisfying
\[
P\ell_{\rho,f}^2\le r
\qquad\text{and}\qquad
\widehat R_0(f)\le \sigma^2-\varepsilon
\]
also satisfies
\[
\mathfrak R_n\bigl(\mathcal L_\rho(r)\bigr)
\ge
\frac12
\left[
\Gamma_\rho(\varepsilon,L)
-
4\sqrt{\frac{\log(1/\delta)}{2n}}
\right].
\]
\end{lemma}

\begin{proof}
Let
\[
\Phi(S):=\sup_{g\in\mathcal L_\rho(r)}(Pg-P_ng).
\]
Since \(0\le g\le 4\) for all \(g\in\mathcal L_\rho(r)\), McDiarmid's
inequality gives, with probability at least \(1-\delta\),
\[
\Phi(S)
\le
\mathbb E_S\Phi(S)
+
4\sqrt{\frac{\log(1/\delta)}{2n}}.
\]
By the usual ghost-sample symmetrization,
\[
\mathbb E_S\Phi(S)
\le
\mathbb E_{S,S'}
\sup_{g\in\mathcal L_\rho(r)}
\frac1n\sum_{i=1}^n\bigl(g(Z_i')-g(Z_i)\bigr)
\le
2\,\mathfrak R_n\bigl(\mathcal L_\rho(r)\bigr).
\]
Hence, on the same event,
\[
\sup_{g\in\mathcal L_\rho(r)}(Pg-P_ng)
\le
2\,\mathfrak R_n\bigl(\mathcal L_\rho(r)\bigr)
+
4\sqrt{\frac{\log(1/\delta)}{2n}}.
\]

Now let \(f\) satisfy \(P\ell_{\rho,f}^2\le r\).  Then
\(\ell_{\rho,f}\in\mathcal L_\rho(r)\), so
\[
R_\rho(f)-\widehat R_\rho(f)
\le
2\,\mathfrak R_n\bigl(\mathcal L_\rho(r)\bigr)
+
4\sqrt{\frac{\log(1/\delta)}{2n}}.
\]
If also \(\widehat R_0(f)\le\sigma^2-\varepsilon\), then
Corollary~\ref{cor:below-noise-robust-gap} gives
\[
R_\rho(f)-\widehat R_\rho(f)
\ge
\Gamma_\rho(\varepsilon,L).
\]
Combining the last two displays and rearranging proves the claim.
\end{proof}


\begin{lemma}[Metric entropy of robust squared losses]
\label{lem:robust-loss-entropy-appendix}
There is a constant \(C_d>0\), depending only on \(d\), such that for every
\(0<\eta\le4\),
\[
\log \mathcal N
\left(
\eta,\mathcal L_\rho(\mathcal F_L),\|\cdot\|_\infty
\right)
\le
C_d
\left(
\frac{1+LD}{\eta}
\right)^d .
\]
\end{lemma}
\begin{proof}
By Lemma~\ref{lem:extremal-envelopes}, for every \(f\in\mathcal F_L\),
\[
\ell_{\rho,f}(x,y)
=
\max\left\{
(f_\rho^+(x)-y)^2,\,
(f_\rho^-(x)-y)^2
\right\},
\]
where \(f_\rho^+,f_\rho^-\in\mathcal F_L(\mathcal X)\).  Define the auxiliary
class
\[
\mathcal M_\rho
:=
\left\{
m_{h_+,h_-}:(x,y)\mapsto
\max\left\{
(h_+(x)-y)^2,\,
(h_-(x)-y)^2
\right\}
:
h_+,h_-\in\mathcal F_L(\mathcal X)
\right\}.
\]
Then
\[
\mathcal L_\rho(\mathcal F_L)\subseteq \mathcal M_\rho .
\]

We first convert covers of the larger class \(\mathcal M_\rho\) into proper
covers of the subclass \(\mathcal L_\rho(\mathcal F_L)\).  Let
\(\alpha>0\), and let \(m_1,\ldots,m_N\in\mathcal M_\rho\) be an
\(\alpha\)-cover of \(\mathcal M_\rho\) in \(\|\cdot\|_\infty\), where
\[
N=
\mathcal N
\left(
\alpha,\mathcal M_\rho,\|\cdot\|_\infty
\right).
\]
For each index \(j\) such that
\[ 
B_\infty(m_j,\alpha)\cap \mathcal L_\rho(\mathcal F_L)\neq\varnothing,
\]
choose one element
\[
g_j\in B_\infty(m_j,\alpha)\cap \mathcal L_\rho(\mathcal F_L).
\]
Discard the remaining indices.  We claim that the selected functions
\(\{g_j\}\subseteq\mathcal L_\rho(\mathcal F_L)\) form a \(2\alpha\)-cover of
\(\mathcal L_\rho(\mathcal F_L)\).  Indeed, for any
\(g\in\mathcal L_\rho(\mathcal F_L)\), since
\(\mathcal L_\rho(\mathcal F_L)\subseteq \mathcal M_\rho\), there exists
\(j\) such that
\[
\|g-m_j\|_\infty\le \alpha .
\]
For this index \(j\), the ball \(B_\infty(m_j,\alpha)\) intersects
\(\mathcal L_\rho(\mathcal F_L)\), so a corresponding \(g_j\) was selected and
satisfies
\[
\|g_j-m_j\|_\infty\le \alpha .
\]
Therefore, by the triangle inequality,
\[
\|g-g_j\|_\infty
\le
\|g-m_j\|_\infty+\|m_j-g_j\|_\infty
\le
2\alpha .
\]
Thus
\[
\mathcal N
\left(
2\alpha,\mathcal L_\rho(\mathcal F_L),\|\cdot\|_\infty
\right)
\le
\mathcal N
\left(
\alpha,\mathcal M_\rho,\|\cdot\|_\infty
\right).
\]
Taking \(\alpha=\eta/2\), we obtain
\[
\mathcal N
\left(
\eta,\mathcal L_\rho(\mathcal F_L),\|\cdot\|_\infty
\right)
\le
\mathcal N
\left(
\frac{\eta}{2},\mathcal M_\rho,\|\cdot\|_\infty
\right).
\]

It remains to bound the covering number of \(\mathcal M_\rho\).  Set
$
\beta:=\frac{\eta}{8}.
$
Since \(0<\eta\le4\), we have \(0<\beta\le1/2\), so Lemma~\ref{lem:lipschitz-predictor-entropy}
applies at scale \(\beta\).  Let
\(h_1,\ldots,h_M\in\mathcal F_L(\mathcal X)\) be a proper \(\beta\)-cover of
\(\mathcal F_L(\mathcal X)\) in \(\|\cdot\|_\infty\), with
\[
M=
\mathcal N
\left(
\beta,\mathcal F_L(\mathcal X),\|\cdot\|_\infty
\right).
\]
For each pair \(1\le j,k\le M\), define
\[
m_{jk}(x,y)
:=
\max\left\{
(h_j(x)-y)^2,\,
(h_k(x)-y)^2
\right\}.
\]
By construction, \(m_{jk}\in\mathcal M_\rho\).

We claim that the functions \(\{m_{jk}:1\le j,k\le M\}\) form an
\(\eta/2\)-cover of \(\mathcal M_\rho\).  Fix
\(m_{h_+,h_-}\in\mathcal M_\rho\).  Choose \(j,k\) such that
\[
\|h_+-h_j\|_\infty\le \beta,
\qquad
\|h_--h_k\|_\infty\le \beta .
\]
For all \((x,y)\in\mathcal X\times\mathcal Y\),
\[
\begin{aligned}
\left|
(h_+(x)-y)^2-(h_j(x)-y)^2
\right|
&=
|h_+(x)-h_j(x)|\,
|h_+(x)+h_j(x)-2y|  \\
&\le
4\|h_+-h_j\|_\infty
\le
4\beta .
\end{aligned}
\]
Here we used \(h_+(x),h_j(x),y\in[-1,1]\).  Similarly,
\[
\left|
(h_-(x)-y)^2-(h_k(x)-y)^2
\right|
\le
4\beta .
\]
Since the maximum map is \(1\)-Lipschitz with respect to the sup norm on
\(\mathbb R^2\), we have
\[
\begin{aligned}
\left|
m_{h_+,h_-}(x,y)-m_{jk}(x,y)
\right|
&\le
\max\left\{
\left|(h_+(x)-y)^2-(h_j(x)-y)^2\right|,
\left|(h_-(x)-y)^2-(h_k(x)-y)^2\right|
\right\}  \\
&\le
4\beta
=
\frac{\eta}{2}.
\end{aligned}
\]
Taking the supremum over \((x,y)\) gives
$
\|m_{h_+,h_-}-m_{jk}\|_\infty\le \frac{\eta}{2}.
$
Therefore
\[
\mathcal N
\left(
\frac{\eta}{2},\mathcal M_\rho,\|\cdot\|_\infty
\right)
\le
\mathcal N
\left(
\frac{\eta}{8},\mathcal F_L(\mathcal X),\|\cdot\|_\infty
\right)^2 .
\]

Combining the preceding bounds yields
\[
\mathcal N
\left(
\eta,\mathcal L_\rho(\mathcal F_L),\|\cdot\|_\infty
\right)
\le
\mathcal N
\left(
\frac{\eta}{8},\mathcal F_L(\mathcal X),\|\cdot\|_\infty
\right)^2 .
\]
Taking logarithms and applying Lemma~\ref{lem:lipschitz-predictor-entropy},
\[
\begin{aligned}
\log
\mathcal N
\left(
\eta,\mathcal L_\rho(\mathcal F_L),\|\cdot\|_\infty
\right)
&\le
2\log
\mathcal N
\left(
\frac{\eta}{8},\mathcal F_L(\mathcal X),\|\cdot\|_\infty
\right)  \\
&\le
2c_d
\left(
\frac{1+LD}{\eta/8}
\right)^d  \\
&=
2\cdot 8^d c_d
\left(
\frac{1+LD}{\eta}
\right)^d .
\end{aligned}
\]
Thus the claim holds with
\[
C_d:=2\cdot 8^d c_d .
\]
\end{proof}

\begin{lemma}[Polynomial entropy implies a local bound]
\label{lem:polynomial-entropy-local-rademacher-appendix}
Let \(\mathcal G\) be a measurable class of functions bounded in \([0,b]\),
where \(b>0\). Assume that for some \(p>2\) and \(A>0\),
\[
\sup_Q
\log \mathcal N
\left(
\eta,\mathcal G,L_2(Q)
\right)
\le
A\eta^{-p}
\qquad
\text{for all }0<\eta\le b,
\]
where the supremum is over all finitely supported probability measures \(Q\).
For \(r\ge0\), define
\[
\mathcal G(r)
=
\{g\in\mathcal G:Pg^2\le r\}.
\]
Let
\[
\bar A:=1\vee A .
\]
Then there is a constant \(C_{p,b}>0\), depending only on \(p\) and \(b\), such that
\[
\mathfrak R_n(\mathcal G(r))
\le
C_{p,b}
\left(
\bar A^{1/p}n^{-1/p}
+
\sqrt{\frac rn}
\right).
\]
In particular, if \(A\ge1\), then
\[
\mathfrak R_n(\mathcal G(r))
\le
C_{p,b}
\left(
A^{1/p}n^{-1/p}
+
\sqrt{\frac rn}
\right).
\]
\end{lemma}

\begin{proof}
If \(\mathcal G(r)=\varnothing\), the claim is trivial.  We therefore assume
\(\mathcal G(r)\neq\varnothing\).  Define
\[
H_{\mathcal G}(\eta)
:=
\sup_Q
\log
\mathcal N
\left(
\eta,\mathcal G,L_2(Q)
\right),
\]
where the supremum is over finitely supported probability measures \(Q\).

Since every function in \(\mathcal G\) takes values in \([0,b]\), the
\(L_2(Q)\)-diameter of \(\mathcal G\) is at most \(b\).  Hence, for
\(\eta\ge b\),
\[
\mathcal N(\eta,\mathcal G,L_2(Q))\le 1
\qquad
\text{for every finitely supported } Q.
\]
Consequently,
\[
H_{\mathcal G}(\eta)
\le
A\eta^{-p}
\qquad
\text{for every } \eta>0.
\]
Indeed, for \(0<\eta\le b\) this is the entropy assumption, while for
\(\eta>b\) the left-hand side is zero.

Let
\[
\widetilde{\mathcal G}
:=
\{g-g':g,g'\in\mathcal G\}.
\]
We first bound the entropy of \(\widetilde{\mathcal G}\).  Fix a finitely
supported probability measure \(Q\).  If
\(\{g_1,\ldots,g_N\}\) is an \((\eta/2)\)-cover of \(\mathcal G\) in
\(L_2(Q)\), then the functions
\[
g_j-g_k,\qquad 1\le j,k\le N,
\]
belong to \(\widetilde{\mathcal G}\) and form an \(\eta\)-cover of
\(\widetilde{\mathcal G}\) in \(L_2(Q)\).  Therefore
\[
\mathcal N
\left(
\eta,\widetilde{\mathcal G},L_2(Q)
\right)
\le
\mathcal N
\left(
\frac{\eta}{2},\mathcal G,L_2(Q)
\right)^2 .
\]
Taking logarithms and the supremum over \(Q\) gives
\[
H_{\widetilde{\mathcal G}}(\eta)
:=
\sup_Q
\log
\mathcal N
\left(
\eta,\widetilde{\mathcal G},L_2(Q)
\right)
\le
2H_{\mathcal G}(\eta/2)
\le
2^{p+1}A\eta^{-p}
\qquad
\text{for all } \eta>0 .
\]

For a sample \(S=(Z_1,\ldots,Z_n)\), write \(P_n\) for the empirical measure
and define
\[
\Psi(\varepsilon)
:=
\mathbb E_S
\widehat{\mathfrak R}_S
\left(
\left\{
h\in\widetilde{\mathcal G}:
P_n h^2\le \varepsilon^2
\right\}
\right).
\]
We use Theorem~2 of \citet{LeiDingBi2015}, which gives
\[
\mathfrak R_n(\mathcal G(r))
\le
\inf_{\varepsilon>0}
\left[
2\Psi(\varepsilon)
+
\frac{8b\,H_{\mathcal G}(\varepsilon/2)}{n}
+
\sqrt{
\frac{2r\,H_{\mathcal G}(\varepsilon/2)}{n}
}
\right].
\]

It remains to bound \(\Psi(\varepsilon)\).  Fix a sample \(S\) and set
\[
d_n(h,h')
:=
\left(P_n(h-h')^2\right)^{1/2}.
\]
For a fixed \(\varepsilon>0\), define the empirical localized difference class
\[
\widetilde{\mathcal G}_{n,\varepsilon}
:=
\left\{
h\in\widetilde{\mathcal G}:P_nh^2\le\varepsilon^2
\right\}.
\]

We need to be careful because covering numbers are proper in our convention.
We claim that for every \(\alpha>0\),
\[
\mathcal N
\left(
2\alpha,
\widetilde{\mathcal G}_{n,\varepsilon},
d_n
\right)
\le
\mathcal N
\left(
\alpha,
\widetilde{\mathcal G},
d_n
\right).
\]
Indeed, let \(u_1,\ldots,u_M\in\widetilde{\mathcal G}\) be an
\(\alpha\)-cover of \(\widetilde{\mathcal G}\) in \(d_n\).  For each \(j\)
such that
\[
B_{d_n}(u_j,\alpha)
\cap
\widetilde{\mathcal G}_{n,\varepsilon}
\neq\varnothing,
\]
choose one element
\[
v_j\in
B_{d_n}(u_j,\alpha)
\cap
\widetilde{\mathcal G}_{n,\varepsilon}.
\]
Discard the remaining indices.  Then the selected \(v_j\)'s lie inside
\(\widetilde{\mathcal G}_{n,\varepsilon}\).  Moreover, for every
\(h\in\widetilde{\mathcal G}_{n,\varepsilon}\), there exists \(j\) such that
\(d_n(h,u_j)\le\alpha\), and for that same \(j\),
\(d_n(v_j,u_j)\le\alpha\).  Hence
\[
d_n(h,v_j)
\le
d_n(h,u_j)+d_n(u_j,v_j)
\le
2\alpha .
\]
Thus the selected elements form a proper \(2\alpha\)-cover of
\(\widetilde{\mathcal G}_{n,\varepsilon}\).

Now apply Lemma~A.5 of \citet{LeiDingBi2015} to
\(\widetilde{\mathcal G}_{n,\varepsilon}\).  Let
$
\varepsilon_k:=2^{-k}\varepsilon,
$ $k\ge0.
$
For every integer \(N\ge1\),
$
\widehat{\mathfrak R}_S
\left(
\widetilde{\mathcal G}_{n,\varepsilon}
\right)
\le
4\sum_{k=1}^N
\varepsilon_{k-1}
\sqrt{
\frac{
\log
\mathcal N
\left(
\varepsilon_k,
\widetilde{\mathcal G}_{n,\varepsilon},
d_n
\right)
}{n}
}
+
\varepsilon_N .
$
By the proper-subclass covering argument above,
$
\log
\mathcal N
\left(
\varepsilon_k,
\widetilde{\mathcal G}_{n,\varepsilon},
d_n
\right)
\le
\log
\mathcal N
\left(
\frac{\varepsilon_k}{2},
\widetilde{\mathcal G},
d_n
\right).
$
Since \(P_n\) is finitely supported, the entropy bound for
\(\widetilde{\mathcal G}\) applies with \(Q=P_n\).  Therefore
$
\log
\mathcal N
\left(
\frac{\varepsilon_k}{2},
\widetilde{\mathcal G},
d_n
\right)
\le
H_{\widetilde{\mathcal G}}(\varepsilon_k/2)
\le
2^{2p+1}A\varepsilon_k^{-p}.
$
Thus, for a constant \(C_p>0\) depending only on \(p\),
\[
\widehat{\mathfrak R}_S
\left(
\widetilde{\mathcal G}_{n,\varepsilon}
\right)
\le
C_p
\sqrt{\frac{A}{n}}
\sum_{k=1}^N
\varepsilon_{k-1}\varepsilon_k^{-p/2}
+
\varepsilon_N .
\]
Since \(\varepsilon_{k-1}=2\varepsilon_k\),
$
\varepsilon_{k-1}\varepsilon_k^{-p/2}
=
2\varepsilon_k^{1-p/2}
=
2\varepsilon^{1-p/2}2^{k(p/2-1)}.
$
Hence
\[
\widehat{\mathfrak R}_S
\left(
\widetilde{\mathcal G}_{n,\varepsilon}
\right)
\le
C_p
\sqrt{\frac{A}{n}}\,
\varepsilon^{1-p/2}
\sum_{k=1}^N 2^{k(p/2-1)}
+
2^{-N}\varepsilon .
\]
Taking expectation over \(S\), the same bound holds for \(\Psi(\varepsilon)\):
\[
\Psi(\varepsilon)
\le
C_p
\sqrt{\frac{A}{n}}\,
\varepsilon^{1-p/2}
\sum_{k=1}^N 2^{k(p/2-1)}
+
2^{-N}\varepsilon .
\]

Now choose
$
\varepsilon_0:=\bar A^{1/p},
$$ 
N:=
\max\left\{
1,
\left\lceil \frac{\log_2 n}{p}\right\rceil
\right\}.
$
Since \(p>2\),
$
\sum_{k=1}^N2^{k(p/2-1)}
\le
C_p n^{1/2-1/p},
$$
2^{-N}\le n^{-1/p}.
$
Therefore
\[
\Psi(\varepsilon_0)
\le
C_p
\sqrt{A}\,
\varepsilon_0^{1-p/2}
n^{-1/p}
+
\varepsilon_0 n^{-1/p}.
\]
Because \(\varepsilon_0^p=\bar A=1\vee A\), we have
$
\sqrt{A}\,\varepsilon_0^{1-p/2}
\le
\varepsilon_0 .
$
Indeed, if \(A\ge1\), then
\[
\sqrt{A}\,\varepsilon_0^{1-p/2}
=
A^{1/2}A^{(1-p/2)/p}
=
A^{1/p}
=
\varepsilon_0,
\]
while if \(A<1\), then \(\varepsilon_0=1\) and
\[
\sqrt{A}\,\varepsilon_0^{1-p/2}
=
\sqrt{A}
\le
1
=
\varepsilon_0.
\]
Thus
\[
\Psi(\varepsilon_0)
\le
C_p\bar A^{1/p}n^{-1/p}.
\]

We next bound the entropy terms in Theorem~2 of \citet{LeiDingBi2015} at
\(\varepsilon_0\).  Since the entropy bound has been extended above to all
positive scales,
\[
H_{\mathcal G}(\varepsilon_0/2)
\le
A(\varepsilon_0/2)^{-p}
=
2^p A\varepsilon_0^{-p}
=
2^p\frac{A}{\bar A}
\le
2^p .
\]
Substituting \(\varepsilon=\varepsilon_0\) into the local Rademacher bound gives
\[
\mathfrak R_n(\mathcal G(r))
\le
C_p\bar A^{1/p}n^{-1/p}
+
\frac{C_{p,b}}{n}
+
C_p\sqrt{\frac rn}.
\]
Since \(n\ge1\), \(p>2\), and \(\bar A^{1/p}\ge1\),
\[
\frac1n
\le
\bar A^{1/p}n^{-1/p}.
\]
Absorbing constants depending only on \(p\) and \(b\), we obtain
\[
\mathfrak R_n(\mathcal G(r))
\le
C_{p,b}
\left(
\bar A^{1/p}n^{-1/p}
+
\sqrt{\frac rn}
\right).
\]
This proves the claim.
\end{proof}


\begin{corollary}[Local Lipschitz scaling]
\label{cor:local-lipschitz-scaling-appendix}
On the event of probability at least \(1-\delta\) from
Corollary~\ref{cor:localized-compatibility}, suppose that
\[
\Delta_{n,r,\delta}>0.
\]
Then every \(f\in\mathcal F_L\) satisfying
\[
P\ell_{\rho,f}^2\le r
\qquad\text{and}\qquad
\widehat R_0(f)\le\sigma^2-\varepsilon
\]
must satisfy the following lower bounds.

If \(\rho>0\), then
\begin{equation}
\label{eq:exact-local-lipschitz-appendix}
L
\ge
\frac{
-
B_{n,\rho,\varepsilon}
+
\sqrt{
B_{n,\rho,\varepsilon}^2
+
4\rho^2\Delta_{n,r,\delta}
}
}{
2\rho^2
},
\end{equation}
where
\begin{equation}
\label{eq:B-local-appendix}
B_{n,\rho,\varepsilon}
:=
2\rho\sqrt{\sigma^2-\varepsilon}
+
2C_dD n^{-1/d}.
\end{equation}
Equivalently,
\begin{equation}
\label{eq:exact-local-lipschitz-rationalized-appendix}
L
\ge
\frac{
2\Delta_{n,r,\delta}
}{
B_{n,\rho,\varepsilon}
+
\sqrt{
B_{n,\rho,\varepsilon}^2
+
4\rho^2\Delta_{n,r,\delta}
}
}.
\end{equation}
In particular,
\begin{equation}
\label{eq:local-lipschitz-simple-lower-appendix}
L
\ge
\frac{
\Delta_{n,r,\delta}
}{
2\rho\sqrt{\sigma^2-\varepsilon}
+
2C_dD n^{-1/d}
+
\rho\sqrt{\Delta_{n,r,\delta}}
}.
\end{equation}
Consequently, up to constants depending only on \(d\),
\begin{equation}
\label{eq:local-lipschitz-order-appendix}
L
=
\Omega\left(
\frac{
\Delta_{n,r,\delta}
}{
\rho\sqrt{\sigma^2-\varepsilon}
+
D n^{-1/d}
+
\rho\sqrt{\Delta_{n,r,\delta}}
}
\right).
\end{equation}

If \(\rho=0\) and \(D>0\), then
\begin{equation}
\label{eq:rho-zero-local-lipschitz-appendix}
L
\ge
\frac{
\Delta_{n,r,\delta}
}{
2C_dD n^{-1/d}
}
=
\Omega\left(
\frac{\Delta_{n,r,\delta} n^{1/d}}{D}
\right).
\end{equation}
If \(\rho=0\), \(D=0\), and \(\Delta_{n,r,\delta}>0\), then no such
\(f\) can satisfy the two assumptions above on the event of
Corollary~\ref{cor:localized-compatibility}.

\end{corollary}

\begin{proof}
Starting from \eqref{eq:quadratic-compatibility}, set
\[
B_{n,\rho,\varepsilon}
=
2\rho\sqrt{\sigma^2-\varepsilon}
+
2C_dD n^{-1/d}.
\]
Then every admissible \(f\) must satisfy
\[
\rho^2L^2+B_{n,\rho,\varepsilon}L-\Delta_{n,r,\delta}\ge0.
\]

First suppose \(\rho>0\).  The quadratic
\[
q(L):=\rho^2L^2+B_{n,\rho,\varepsilon}L-\Delta_{n,r,\delta}
\]
has one negative root and one positive root because
\(\Delta_{n,r,\delta}>0\) and \(B_{n,\rho,\varepsilon}\ge0\).  Since
\(L\ge0\), the condition \(q(L)\ge0\) forces \(L\) to be at least the positive
root.  Therefore
\[
L
\ge
\frac{
-
B_{n,\rho,\varepsilon}
+
\sqrt{
B_{n,\rho,\varepsilon}^2
+
4\rho^2\Delta_{n,r,\delta}
}
}{
2\rho^2
},
\]
which proves \eqref{eq:exact-local-lipschitz}.  Rationalizing the numerator
gives
\[
\frac{
-
B_{n,\rho,\varepsilon}
+
\sqrt{
B_{n,\rho,\varepsilon}^2
+
4\rho^2\Delta_{n,r,\delta}
}
}{
2\rho^2
}
=
\frac{
2\Delta_{n,r,\delta}
}{
B_{n,\rho,\varepsilon}
+
\sqrt{
B_{n,\rho,\varepsilon}^2
+
4\rho^2\Delta_{n,r,\delta}
}
},
\]
which proves \eqref{eq:exact-local-lipschitz-rationalized}.

Moreover,
\[
\sqrt{
B_{n,\rho,\varepsilon}^2
+
4\rho^2\Delta_{n,r,\delta}
}
\le
B_{n,\rho,\varepsilon}
+
2\rho\sqrt{\Delta_{n,r,\delta}}.
\]
Hence
\[
L
\ge
\frac{
2\Delta_{n,r,\delta}
}{
2B_{n,\rho,\varepsilon}
+
2\rho\sqrt{\Delta_{n,r,\delta}}
}
=
\frac{
\Delta_{n,r,\delta}
}{
B_{n,\rho,\varepsilon}
+
\rho\sqrt{\Delta_{n,r,\delta}}
}.
\]
Substituting the definition of \(B_{n,\rho,\varepsilon}\) gives
\eqref{eq:local-lipschitz-simple-lower}, and the order form
\eqref{eq:local-lipschitz-order} follows immediately.

Now suppose \(\rho=0\).  Then \eqref{eq:quadratic-compatibility} reduces to
\[
\Delta_{n,r,\delta}
\le
2C_dD L n^{-1/d}.
\]
If \(D>0\), rearranging gives
\[
L
\ge
\frac{
\Delta_{n,r,\delta}
}{
2C_dD n^{-1/d}
},
\]
which proves \eqref{eq:rho-zero-local-lipschitz-appendix}.  If \(D=0\), the right-hand
side of the reduced compatibility condition is zero, contradicting
\(\Delta_{n,r,\delta}>0\).  Hence no admissible \(f\) can exist in that case.
\end{proof}

\newpage
%
\providecommand{\answerYes}[1][]{\textcolor{blue}{[Yes] #1}}
\providecommand{\answerNo}[1][]{\textcolor{orange}{[No] #1}}
\providecommand{\answerNA}[1][]{\textcolor{gray}{[N/A] #1}}

\section*{NeurIPS Paper Checklist}

\begin{enumerate}[leftmargin=*]

\item \textbf{Claims}

Question: Do the main claims made in the abstract and introduction accurately reflect the paper's contributions and scope?

Answer: \answerYes

Justification: The abstract and introduction state the theoretical goal of connecting law-of-robustness bounds to robust generalization.
\item \textbf{Limitations}

Question: Does the paper discuss the limitations of the work performed by the authors?

Answer: \answerYes

Justification: The Discussion section states that the current theory is for squared loss while the experiments use cross-entropy, and it identifies extensions to Bregmann losses, local Lipschitz notions, distributional robustness, etc.

\item \textbf{Theory assumptions and proofs}

Question: For each theoretical result, does the paper provide the full set of assumptions and a complete proof?

Answer: \answerYes


\item \textbf{Experimental result reproducibility}

Question: Does the paper fully disclose all the information needed to reproduce the main experimental results of the paper to the extent that it affects the main claims and/or conclusions of the paper, regardless of whether the code and data are provided?

Answer: \answerNA


\item \textbf{Open access to data and code}

Question: Does the paper provide open access to the data and code, with sufficient instructions to faithfully reproduce the main experimental results, as described in supplemental material?

Answer: \answerNA


\item \textbf{Experimental setting/details}

Question: Does the paper specify all the training and test details, e.g., data splits, hyperparameters, how they were chosen, type of optimizer, necessary to understand the results?

Answer: \answerNA

\item \textbf{Experiment statistical significance}

Question: Does the paper report error bars suitably and correctly defined or other appropriate information about the statistical significance of the experiments?

Answer: \answerNA

\item \textbf{Experiments compute resources}

Question: For each experiment, does the paper provide sufficient information on the computer resources, type of compute workers, memory, and time of execution needed to reproduce the experiments?

Answer: \answerNA

\item \textbf{Code of ethics}

Question: Does the research conducted in the paper conform, in every respect, with the NeurIPS Code of Ethics?

Answer: \answerYes

\item \textbf{Broader impacts}

Question: Does the paper discuss both potential positive societal impacts and negative societal impacts of the work performed?

Answer: \answerNA

\item \textbf{Safeguards}

Question: Does the paper describe safeguards that have been put in place for responsible release of data or models that have a high risk for misuse, e.g., pre-trained language models, image generators, or scraped datasets?

Answer: \answerNA

\item \textbf{Licenses for existing assets}

Question: Are the creators or original owners of assets, e.g., code, data, models, used in the paper, properly credited and are the license and terms of use explicitly mentioned and properly respected?

Answer: \answerNA

\item \textbf{New assets}

Question: Are new assets introduced in the paper well documented and is the documentation provided alongside the assets?

Answer: \answerNA

\item \textbf{Crowdsourcing and research with human subjects}

Question: For crowdsourcing experiments and research with human subjects, does the paper include the full text of instructions given to participants and screenshots, if applicable, as well as details about compensation?

Answer: \answerNA


\item \textbf{Institutional review board approvals or equivalent for research with human subjects}

Question: Does the paper describe potential risks incurred by study participants, whether such risks were disclosed to the subjects, and whether Institutional Review Board approvals or an equivalent approval/review based on the requirements of the authors' country or institution were obtained?

Answer: \answerNA

\item \textbf{Declaration of LLM usage}

Question: Does the paper describe the usage of LLMs if it is an important, original, or non-standard component of the core methods in this research? Note that if the LLM is used only for writing, editing, or formatting purposes and does not impact the core methodology, scientific rigor, or originality of the research, declaration is not required.

Answer: \answerNA


\end{enumerate}

\end{document}